\pdfoutput=1
\documentclass[11pt]{article} 
\usepackage{times}
\usepackage[letterpaper, left=1in, right=1in, top=1in, bottom=1in]{geometry}

\usepackage[utf8]{inputenc} 
\usepackage[T1]{fontenc}    
\usepackage{url}            
\usepackage{booktabs}       
\usepackage{amsfonts}       
\usepackage{nicefrac}       
\usepackage{microtype}      
\usepackage{mkolar_definitions}
\usepackage{xspace}
\usepackage{amsmath}
\usepackage{algorithm,algorithmic}
\usepackage{color}
\usepackage{enumitem}
\usepackage{comment}
\usepackage{bm}
\usepackage{graphicx}

\usepackage{apptools}
\usepackage[page, header]{appendix}
\usepackage{titletoc}

\usepackage[scaled=.9]{helvet}

\usepackage{url}
\usepackage{authblk}
\usepackage{caption}

\usepackage[dvipsnames]{xcolor}
\usepackage[colorlinks=true, linkcolor=blue, citecolor=blue]{hyperref}

\newcommand{\trans}{P}
\newcommand{\version}{arxiv}

\newcommand{\ud}{\mathrm{d}}
\newcommand{\Gu}{\Gcal_{\textrm{up}}}
\newcommand{\veps}{\varepsilon}
\newcommand{\conf}{\mathrm{conf}}

\usepackage{natbib}
\bibpunct{(}{)}{;}{a}{,}{,}

\newcount\Comments  
\definecolor{darkgreen}{rgb}{0,0.5,0}
\definecolor{darkred}{rgb}{0.7,0,0}
\definecolor{teal}{rgb}{0.3,0.8,0.8}
\definecolor{orange}{rgb}{1.0,0.5,0.0}
\definecolor{purple}{rgb}{0.8,0.0,0.8}
\newcommand{\kibitz}[2]{\ifnum\Comments=1{\textcolor{#1}{\textsf{\footnotesize #2}}}\fi}

\newcommand{\defeq}{\triangleq}
\newcommand{\Otilde}{\tilde{O}}

\title{Optimism in Reinforcement Learning with Generalized Linear Function Approximation}
\date{}
\author{Yining Wang\thanks{yining.wang@warrington.ufl.edu,  $^{\dagger}$ruosongw@andrew.cmu.edu,$^{\ddagger}$ssdu@ias.edu,$^{\mathsection}$akshaykr@microsoft.com}}
\affil{University of Florida, Gainsville, FL}
\author{Ruosong Wang$^{\dagger}$}
\affil{Carnegie Mellon University, Pittsburgh, PA}
\author{Simon S. Du$^{\ddagger}$}
\affil{Institute for Advanced Studies, Princeton, NJ}
\author{Akshay Krishnamurthy$^{\mathsection}$}
\affil{Microsoft Research, New York, NY}

\begin{document}

\maketitle

\begin{abstract}
We design a new provably efficient algorithm for episodic
reinforcement learning with generalized linear function
approximation. We analyze the algorithm under a new expressivity
assumption that we call ``optimistic closure,'' which is strictly
weaker than assumptions from prior analyses for the linear
setting. With optimistic closure, we prove that our algorithm enjoys a
regret bound of $\Otilde(\sqrt{d^3 T})$ where $d$ is the
dimensionality of the state-action features and $T$ is the number of
episodes. This is the first statistically and computationally
efficient algorithm for reinforcement learning with generalized linear
functions.
\end{abstract}

\section{Introduction}

We study episodic reinforcement learning problems with infinitely
large state spaces, where the agent must use function approximation to
generalize across states while simultaneously engaging in strategic
exploration. Such problems form the core of modern empirical/deep-RL,
but relatively little work focuses on exploration, and even fewer
algorithms enjoy strong sample efficiency guarantees.

On the theoretical side, classical sample efficiency results from the
early 00s focus on ``tabular'' environments with small finite state
spaces~\citep{kearns2002near,brafman2002r,strehl2006pac}, but as these
methods scale with the number of states, they do not address problems
with infinite or large state spaces. While this classical work has
inspired practically effective approaches for large state
spaces~\citep{bellemare2016unifying,osband2016deep,tang2017exploration},
these methods do not enjoy sample efficiency guarantees.  More recent
theoretical progress has produced provably sample efficient algorithms
for complex environments, but many of these algorithms are relatively
impractical~\citep{krishnamurthy2016pac,jiang2017contextual}. In
particular, these methods are computationally inefficient or rely
crucially on strong dynamics assumptions~\citep{du2019provably}.

In this paper, with an eye toward practicality, we study a simple
variation of Q-learning, where we approximate the optimal Q-function
with a generalized linear model. The algorithm is appealingly simple:
collect a trajectory by following the greedy policy corresponding to
the current model, perform a dynamic programming back-up to update the
model, and repeat.  The key difference over traditional
Q-learning-like algorithms is in the dynamic programming step. Here we
ensure that the updated model is \emph{optimistic} in the sense that
it always overestimates the optimal Q-function. This optimism is
essential for our guarantees.

Optimism in the face of uncertainty is a well-understood and powerful
algorithmic principle in short-horizon (e.g,. bandit) problems, as
well as in tabular reinforcement
learning~\citep{azar2017minimax,dann2017unifying,jin2018q}. With
linear function approximation,~\citet{yang2019reinforcement}
and~\citet{jin2019provably} show that the optimism principle can also
yield provably sample-efficient algorithms, when the environment
dynamics satisfy a certain linearity properties. Their assumptions are
always satisfied in tabular problems, but are somewhat unnatural in
settings where function approximation is required. Moreover as these
assumptions are directly on the dynamics, it is unclear how their
analysis can accommodate other forms of function approximation,
including generalized linear models.

In the present paper, we replace explicit dynamics assumptions with
expressivity assumptions on the function approximator, and, by
analyzing a similar algorithm to~\citet{jin2019provably}, we show that
the optimism principle succeeds under these strictly weaker
assumptions.  More importantly, the relaxed assumption facilitates
moving beyond linear models, and we demonstrate this by providing the
first practical and provably efficient RL algorithm with generalized
linear function approximation.

\section{Preliminaries}

We consider episodic reinforcement learning in a finite-horizon markov
decision process (MDP) with possibly infinitely large state space
$\Scal$, finite action space $\Acal$, initial distribution
$\mu \in \Delta(\Scal)$, transition operator
$\trans: \Scal \times \Acal \to \Delta(\Scal)$, reward function
$R: \Scal \times \Acal \to \Delta([0,1])$ and horizon $H$.  The agent
interacts with the MDP in episodes and, in each episode, a trajectory
$(s_1,a_1,r_1,s_2,a_2,r_2,\ldots,s_H,a_H,r_H)$ is generated where
$s_1 \sim \mu$, for $h > 1$ we have $s_h \sim \trans(\cdot \mid
s_{h-1},a_{h-1})$, $r_h \sim R(s_h,a_h)$, and actions $a_{1:H}$ are
chosen by the agent. For normalization, we assume that $\sum_{h=1}^H
r_h \in [0,1]$ almost surely.


A (deterministic, nonstationary) policy $\pi=(\pi_1,\cdots,\pi_H)$
consists of $H$ mappings $\pi_h: \Scal \to \Acal$, where $\pi_h(s_h)$
denotes the action to be taken at time point $h$ if at state $s_h\in\Scal$
The \emph{value} function for a policy $\pi$ is a collection of
functions $(V_1^\pi,\ldots,V_H^\pi)$ where $V_h^\pi : \Scal \to \RR$
is the expected future reward the policy collects if it starts in a
particular state at time point $h$. Formally,
\begin{align*}
V_h^\pi(s) \defeq \EE\sbr{\sum_{h'=h}^H r_{h'} \mid s_h=s, a_{h:H}\sim\pi }.
\end{align*}
The value for a policy $\pi$ is simply
$V^\pi \defeq \EE_{s_1 \sim \mu} \sbr{V_1^\pi(s_1)}$, and the optimal
value is $V^\star \defeq \max_{\pi} V^\pi$, where the maximization is
over all nonstationary policies. The typical goal is to find an
approximately optimal policy, and in this paper, we measure
performance by the regret accumulated over $T$ episodes,
\begin{align*}
\textrm{Reg}(T) \defeq T V^\star - \EE \sbr{\sum_{t=1}^T \sum_{h=1}^H r_{h,t}}.
\end{align*}
Here $r_{h,t}$ is the reward collected by the agent at time point
$h$ in the $t^{\textrm{th}}$ episode. We seek algorithms with regret
that is sublinear in $T$, which demonstrates the agent's ability to
act near-optimally.


\subsection{Q-values and  function approximation}

For any policy $\pi$, the state-action value function, or the
$Q$-function is a sequence of mappings $Q^\pi =
(Q_1^\pi,\ldots,Q_H^\pi)$ where $Q_h^\pi: \Scal \times \Acal \to \RR$
is defined as
\begin{align*}
Q^\pi_h(s,a) \defeq \EE\sbr{\sum_{h'=h}^H r_{h'} \mid
s_h=s,a_h=a, a_{h+1:H} \sim \pi}.
\end{align*}
The optimal $Q$-function is $Q^\star_h \defeq Q_h^{\pi^\star}$ where
$\pi^\star \defeq \argmax_{\pi} V^\pi$ is the optimal policy. 

In the value-based function approximation setting, we use a function
class $\Gcal$ to model $Q^\star$. In this paper, we always take
$\Gcal$ to be a class of generalized linear models (GLMs), defined as
follows: Let $d \in \NN$ be a dimensionality parameter and let
$\BB_d \defeq \cbr{ x \in \RR^d: \nbr{x}_2 \leq 1}$ be the $\ell_2$
ball in $\RR^d$. 

\begin{definition}
For a \emph{known} feature map $\phi: \Scal \times \Acal\to \BB_d$ and
a \emph{known} link function $f: [-1,1] \mapsto [-1,1]$ the class
of \emph{generalized linear models} is $\Gcal \defeq \{ (s,a) \mapsto
f(\inner{\phi(s,a)}{\theta}) : \theta \in \BB_d\}$.
\end{definition}

As is standard in the
literature~\citep{filippi2010parametric,li2017provably}, we assume the
link function satisfies certain regularity conditions.


\begin{assum}
$f(\cdot)$ is either monotonically increasing or decreasing.
Furthermore, there exist absolute constants $0<\kappa<K<\infty$ and
$M<\infty$ such that $\kappa\leq |f'(z)|\leq K$ and $|f''(z)|\leq M$
for all $|z|\leq 1$.
\label{assum:glm-regularity}
\end{assum}

For intuition, two example link functions are the identity map $f(z) =
z$ and the logistic map $f(z) = 1/(1+e^{-z})$ with bounded $z$. It is easy to verify
that both of these maps satisfy~\pref{assum:glm-regularity}.

\subsection{Expressivity assumptions: realizability and optimistic closure}

To obtain sample complexity guarantees that scale polynomially with
problem parameters in the function approximation setting, it is
necessary to posit expressivity assumptions on the function class
$\Gcal$~\citep{krishnamurthy2016pac,du2019good}. The weakest such
condition is \emph{realizability}, which posits that the optimal $Q$
function is in $\Gcal$, or at least well-approximated by
$\Gcal$. Realizability alone suffices for provably efficient
algorithms in the ``contextual bandits'' setting where
$H=1$~\citep{li2017provably,filippi2010parametric,abbasi2011improved},
but it does not seem to be sufficient when $H>1$. Indeed in these
settings it is common to make stronger expressivity
assumptions~\citep{chen2019information,yang2019reinforcement,jin2019provably}.

Following these works, our main assumption is a closure property of the
\emph{Bellman update} operator $\Tcal_h$. This operator has type
$\Tcal_h: (\Scal\times\Acal \to \RR) \to (\Scal\times\Acal \to \RR)$
and is defined for all $s \in \Scal, a \in \Acal$ as
\begin{align*}
\Tcal_h(Q) (s,a) &\defeq \EE\sbr{ r_h + V_Q(s_{h+1}) \mid s_h=s,a_h=a },
\ifthenelse{\equal{\version}{arxiv}}{\qquad V_Q(s) \defeq \max_{a\in\Acal} Q(s,a).}{\\V_Q(s) &\defeq \max_{a\in\Acal} Q(s,a).}
\end{align*}
The Bellman update operator for time point $H$ is simply
$\Tcal_H(Q)(s,a) \defeq \EE\sbr{r_H \mid s_H=s,a_H=a}$, which is
degenerate.  To state the assumption, we must first define the
enlarged function class $\Gu$. For a $d \times d$ matrix $A$, $A
\succeq 0$ denotes that $A$ is positive semi-definite. For a positive
semi-definite matrix $A$, $\nbr{A}_{\textrm{op}}$ is the matrix
operator norm, which is just the largest eigenvalue, and $\nbr{x}_A
\defeq \sqrt{x^\top A x}$ is the matrix Mahalanobis seminorm. For a fixed
constant $\Gamma \in \RR_+$ that we will set to be polynomial in $d$
and $\log(T)$, define
\begin{align*}
\ifthenelse{\equal{\version}{arxiv}}{
\Gu \defeq \cbr{ (s,a) \mapsto \min\cbr{1,
f(\inner{\phi(s,a)}{\theta}) + \gamma\nbr{\phi(s,a)}_{A}} : \theta \in \BB_d, 0 \leq \gamma \leq \Gamma, A \succeq 0, \nbr{A}_{\textrm{op}} \leq 1},
}{
\Gu \defeq \bigg\{ (s,a) \mapsto 1 \wedge
f(\inner{\phi(s,a)}{\theta}) + \gamma\nbr{\phi(s,a)}_{A}:\\
\theta \in \BB_d, A \succeq 0, \nbr{A}_{\textrm{op}} \leq 1\bigg\},
}
\end{align*}
\ifthenelse{\equal{\version}{icml}}{Here we use $a \wedge b \defeq \min\{a,b\}$.}{}
The class $\Gu$ contains $\Gcal$ in addition to all
possible upper confidence bounds that arise from solving least squares
regression problems using the class $\Gcal$. We now state our main
expressivity assumption, which we call \emph{optimistic closure}.


\begin{assum}[Optimistic closure]
For any $1 \leq h < H$ and $g \in \Gu$, we have $\Tcal_h(g) \in \Gcal$.
\label{assum:completeness}
\end{assum}

In words, when we perform a Bellman backup on any upper confidence
bound function for time point $h+1$, we obtain a generalized linear
function at time $h$. While this property seems quite strong, we note
that related closure-type assumptions are common in the literature,
discussed in detail in~\pref{sec:related}. More importantly, we will
see shortly that optimistic closure is actually \emph{strictly weaker}
than previous assumptions used in our RL setting where exploration is
required.  Before turning to these discussions, we mention two basic
properties of optimistic closure. The proofs are deferred
to~\pref{app:simple}.

\begin{fact}[Optimistic closure and realizability]
\label{fact:realizability}
Optimistic closure implies that $Q^\star \in \Gcal$ (realizability).
\end{fact}

\begin{fact}[Optimistic closure in tabular settings]
\label{fact:tabular}
If $\Scal$ is finite and $\phi(s,a) = e_{s,a}$ is the standard-basis
feature map, then under~\pref{assum:glm-regularity} we have optimistic
closure.
\end{fact}

\subsection{Related work}
\label{sec:related}

The majority of the theoretical results for reinforcement learning
focus on the \emph{tabular} setting where the state space is finite
and sample complexities scaling polynomially with $|\Scal|$ are
tolerable~\citep{kearns2002near,brafman2002r,strehl2006pac}. Indeed,
by now there are a number of algorithms that achieve strong guarantees
in these
settings~\citep{dann2017unifying,azar2017minimax,jin2018q,simchowitz2019non}. Via~\pref{fact:tabular},
our results apply to this setting, and indeed our algorithm can be
viewed as a generalization of an existing tabular
algorithm~\citep{azar2017minimax} to the function approximation
setting.\footnote{The description of the algorithm looks quite
  different from that of~\citet{azar2017minimax}, but via an
  equivalence between model-free methods with experience replay and
  model-based methods~\citep{fujimoto2018off}, they are indeed quite
  similar.}

Turning to the function approximation setting, several other results
concern function approximation in setings where exploration is not an
issue, including the infinite-data
regime~\citep{munos2003error,farahmand2010error} and ``batch RL''
settings where the agent does not control the data-collection
process~\citep{munos2008finite,antos2008learning,chen2019information}.
While the settings differ, all of these results require that the
function class satisfy some form of (approximate) closure with respect
to the Bellman operator. These results therefore provide motivation
for our optimistic closure assumption.

A recent line of work studies function approximation in settings where
the agent must explore the
environment~\citep{krishnamurthy2016pac,jiang2017contextual,du2019provably}. The
algorithms developed here can accommodate function classes beyond
generalized linear models, but they are still relatively impractical
and the more practical ones require strong dynamics assumptions~\citep{du2019provably}. In
contrast, our algorithm is straightforward to implement and does not
require any explicit dynamics assumption. As such, we view these
results as complementary to our own.

Lastly, we mention the recent results of~\citet{yang2019reinforcement}
and~\citet{jin2019provably}, which are most closely related to our
work. Both papers
study MDPs with certain linear dynamics assumptions (what they call
the Linear MDP assumption) and use linear function approximation to
obtain provably efficient algorithms. Our algorithm is almost
identical to that of~\citet{jin2019provably}, but, as we will see,
optimistic closure is strictly weaker than their Linear MDP assumption
(which is strictly weaker than the assumption
of~\citet{yang2019reinforcement}). Further, and perhaps more
importantly, our results enable approximation with GLMs, which are
incompatible with the Linear MDP structure. Hence, the present paper
can be seen as a significant generalization of these recent results.


\section{On optimistic closure}

For a more detailed comparison to the recent work results
of~\citet{yang2019reinforcement} and~\citet{jin2019provably}, we
define the linear MPD model studied in the latter work.

\begin{definition}
An MDP is said to be a \emph{linear MDP} if there exist known feature map
$\psi: \Scal \times\Acal \to \RR^d$, unknown signed measures $\mu:\Scal \to
\RR^d$, and an unknown vector $\eta \in \RR^d$ such that (1)
$\trans(s'|s,a)=\langle\psi(s,a),\mu(s')\rangle$ holds for all states
$s,s'$ and actions $a$, and (2) $\EE[r \mid s,a] = \langle
\psi(s,a),\eta\rangle$.
\label{def:linear-mdp}
\end{definition}


Linear MDPs are studied by~\citet{jin2019provably}, who establish a
$\sqrt{T}$-type regret bound for an optimistic algorithm. This
assumption already subsumes that of~\citet{yang2019reinforcement}, and
related assumptions also appear elsewhere in the
literature~\citep{bradtke1996linear,melo2007q}.  In this section, we
show that~\pref{assum:completeness} is a strictly weaker than assuming
the environment is a linear MDP.


\begin{proposition}
If an MDP is linear then~\pref{assum:completeness} holds with
$\Gcal = \{ (s,a) \mapsto \langle w,\psi(s,a)\rangle: w \in \BB_d\}$ .
\end{proposition}
\begin{proof}
The result is implicit in~\citet{jin2019provably}, and we include the
proof for completeness. For any function $g$, observe that owing to the
linear MDP property
\begin{align*}
\Tcal_h(g)(s,a) = \EE\sbr{r + \max_{a'} g(s',a') \mid s,a} = \inner{\psi(s,a)}{\eta} + \int \langle \psi(s,a), \mu(s')\rangle \max_{a'}g(s',a') \ud s',
\end{align*}
which is clearly a linear function in $\psi(s,a)$. Hence for any
function $g$, which trivially includes the optimistic functions,
we have $\Tcal_h(g) \in \Gcal$.
\end{proof}

Thus the linear MDP assumption is stronger
than~\pref{assum:completeness}. Next, we show that it is strictly
stronger.
\begin{proposition}
There exists an MDP with $H=2$, $d=2$, $|\Acal|=2$ and
$|\Scal|=\infty$ such that~\pref{assum:completeness} is satisfied, but
the MDP is not a linear MDP.
\end{proposition}

Thus we have that optimistic closure is strictly weaker than the linear
MDP assumption from~\citet{jin2019provably}. Thus, our results
strictly generalize theirs.

\begin{proof}
In this proof we fix the link function $f(z) =z$.
We first construct the MDP. 
We set the action space $\mathcal{A} = \{a_1, a_2\}$.
We use $e_i$ to denote the $i^{\textrm{th}}$ standard basis element, and let $x =
(0.1/\Gamma,0.1/\Gamma)$ be a fixed vector where $\Gamma$ appears in
the construction $\Gu$. 
Recall that $s_1$ is the first state in each trajectory.
In our example, for all $a \in \mathcal{A}$, $\phi(s_1, a)$ is sampled uniformly at random from the set $\{\alpha
e_1 + (1-\alpha)e_2: \alpha \in [0,1]\}$. The transition rule is
deterministic:
\begin{align*}
\phi(s_2,a_1) = \phi(s_2,a_2) = \alpha x \textrm{ if } \phi(s_1,a) = \alpha e_1 + (1-\alpha)e_2.
\end{align*}
Moreover, for the reward function, $R(s_1, a) = 0$ and $R(s_2, a) = 0.1\alpha /\Gamma$.

We first show that the Linear MDP property does not hold for the constructed MDP and the given feature
map $\phi$. 
Let $s_1^{(1)}$ be the state with $\phi(s_1^{(1)}, a) = e_2$ and $s_1^{(2)}$ be the state with $\phi(s_1^{(2)}, a) = e_1$.
Notice that we deterministically transition from $s_1^{(1)}$ to a state $s_2^{(1)}$ with $\phi(s_2^{(1)}, a) = 0$,
and 
we deterministically transition from $s_1^{(2)}$ to a state $s_2^{(2)}$ with $\phi(s_2^{(2)}, a) = x$, which already fixes the whole transition operator under the linear MDP assumption. 
Thus, under the linear MDP assumption, we
must therefore have a randomized transition for any state $s_1$ with $\phi(s_1,a)
= \alpha e_1 + (1-\alpha)e_2$ where $\alpha \in (0,1)$. This
contradicts the fact that our constructed MDP has deterministic
transitions everywhere, so the linear MDP cannot hold.

We next show that~\pref{assum:completeness} holds. Consider an
arbitrary optimistic $Q$ estimate of the form $g(z) = \min\{1,z^\top
\theta + \gamma \sqrt{z^\top A z}\} \in \Gu$. 
Notice that for $x = (0.1/\Gamma,0.1/\Gamma)$, we always have that $x^\top \theta
+ \gamma \sqrt{x^\top A x} \leq 1$ for any $\theta \in \BB_d$ and $A$
with $\nbr{A}_{\textrm{op}} \leq 1$. 
Moreover, for all $s_2$, i.e., the second state in the trajectory, we always have $\phi(s_2, a) = \alpha x$ for some $\alpha \in [0, 1]$.
Hence we can ignore the first
term in the minimum, and, by direct calculation, we have that when
$\phi(s,a) = \alpha e_1 + (1-\alpha)e_2$:
\begin{align*}
\Tcal_1(g)(s,a) = \alpha x^\top \theta+\gamma\sqrt{\alpha^2x^\top A x} = \alpha (x^\top \theta + \gamma\sqrt{x^\top A x}) = \alpha c_0.
\end{align*}
Hence we can write $\Tcal_1(g) = \inner{\phi(s,a)}{(c_0,0)}$, which verifies~\pref{assum:completeness}.
\end{proof}

\section{Algorithm and main result}
\label{sec:main}

We now turn to our main results. We study a least-squares dynamic
programming style algorithm that we call LSVI-UCB, with pseudocode
presented in~\pref{alg:lsvi-ucb}. The algorithm is nearly identical to
the algorithm proposed by~\citet{jin2019provably} with the same
name. As such, it should be considered as a generalization.

\begin{algorithm*}[t]
\caption{The LSVI-UCB algorithm with generalized linear function approximation.}
\begin{algorithmic}[1]
\STATE Initialize estimates $\bar Q_{h,0} \equiv 1$ for all $h\leq H$ and $\bar Q_{H+1,t}\equiv 0$ for all $1 \leq t \leq T$;
\STATE Set $\gamma = C K\kappa^{-1}\sqrt{1+M+K+d^2\ln((1+K+\Gamma)TH)}$ for a universal constant $C$;
\FOR{$t=1,2,\cdots,T$}
	\STATE Commit to policy $\hat{\pi}_{h,t}(s) \defeq \argmax_{a\in\mathcal A} \bar Q_{h,t-1}(s,a)$;
	\STATE Use policy $\hat{\pi}_{\cdot,t}$ to collect one trajectory $\{(s_{h,t},a_{h,t},r_{h,t})\}_{h=1}^H$;
	\FOR{$h=H,H-1,\cdots,1$}
		\STATE Compute $x_{h,\tau} \defeq \phi(s_{h,\tau},a_{h,\tau})$ and $y_{h,\tau} \defeq r_{h,\tau} + \max_{a'\in\Acal}\bar{Q}_{h+1,t}(s_{h+1,\tau},a')$ for all $\tau\leq t$;
		\STATE Compute ridge estimate
\begin{align}
\hat{\theta}_{h,t} \defeq \argmin_{\|\theta\|_2\leq 1} \sum_{\tau\leq t}( y_{h,\tau}-f(\langle x_{h,\tau},\theta\rangle))^2; \label{eq:ridge}
\end{align}
		\STATE Compute $\Lambda_{h,t} \defeq \sum_{\tau\leq t}x_{h,\tau}x_{h,\tau}^\top+I$;\label{line:cov}
		\STATE Construct $\bar{Q}_{h,t}(s,a) \defeq \min\cbr{1, f(\phi(s,a)^\top\hat{\theta}_{h,t})+\gamma \nbr{\phi(s,a)}_{\Lambda_{h,t}^{-1}}}$;
	\ENDFOR
\ENDFOR
\end{algorithmic}
\label{alg:lsvi-ucb}
\end{algorithm*}

The algorithm uses dynamic programming to maintain optimistic $Q$
function estimates $\{\bar{Q}_{h,t}\}_{h \leq H,t \leq T}$ for each
time point $h$ and each episode $t$. In the $t^{\textrm{th}}$ episode,
we use the previously computed estimates to define the greedy policy
$\hat{\pi}_{h,t}(\cdot) \defeq \argmax_{a \in \Acal}
\bar{Q}_{h,t-1}(\cdot,a)$, which we use to take actions for the
episode. Then, with all of the trajectories collected so far, we
perform a dynamic programming update, where the main per-step
optimization problem is~\pref{eq:ridge}. Starting from time point $H$,
we update our $Q$ function estimates by solving constrained least
squares problems using our class of GLMs. At time point $H$, the
covariates are $\{\phi(s_{H,\tau},a_{H,\tau})\}_{\tau \leq t}$, and
the regression targets are simply the immediate rewards
$\{r_{H,\tau}\}_{\tau \leq t}$. For time points $h < H$, the
covariates are defined similarly as
$\{\phi(s_{h,\tau},a_{h,\tau})\}_{\tau \leq t}$ but the regression
targets are defined by inflating the learned $Q$ function for time
point $h+1$ by an \emph{optimism bonus}.

In detail the least squares problem for time point $h+1$ yields a
parameter $\hat{\theta}_{h+1,t}$ and we also form the second moment
matrix of the covariates $\Lambda_{h+1,t}$. Using these, we define the
optimistic $Q$ function $\bar{Q}_{h+1,t}(s,a) \defeq \min\cbr{1,
  f(\langle \phi(s,a), \hat{\theta}_{h+1,t}\rangle) + \gamma
  \nbr{\phi(s,a)}_{\Lambda_{h+1,t}^{-1}}}$. In our analysis, we verify
that $\bar{Q}_{h+1,t}$ is optimistic in the sense that it
over-estimates $Q^\star$ everywhere. Then, the regression targets for
the least squares problem at time point $h$ are $r_{h,\tau} + \max_{a'
  \in \Acal} \bar{Q}_{h+1,t}(s_{h+1,\tau},a')$, which is a natural
stochastic approximation to the Bellman backup
of $\bar{Q}_{h+1,t}$. Applying this update backward from time point $H$
to $1$, we obtain the $Q$-function estimates that can be used in the
next episode.

The main conceptual difference between~\pref{alg:lsvi-ucb} and the
algorithm of~\citet{jin2019provably} is that we allow non-linear function
approximation with GLMs, while they consider only linear models. On a
more technical level, we use constrained least squares for our dynamic
programming backup which we find easier to analyze, while they use the
ridge regularized version.

On the computational side, the algorithm is straightforward to
implement, and, depending on the link function $f$, it can be easily
shown to run in polynomial time. For example, when $f$ is the identity
map,~\pref{eq:ridge} is standard least square ridge regression, and by
using the Sherman-Morrison formula to amortize matrix inversions, we
can see that the running time is $\order\rbr{d^2|\Acal|HT^2}$. The
dominant cost is evaluating the optimism bonus when computing the
regression targets. In practice, we can use an epoch schedule or
incremental optimization algorithms for updating $\bar{Q}$ for an even
faster algorithm. Of course, with modern machine learning libraries,
it is also straightforward to implement the algorithm with a non-trivial
link function $f$, even though~\pref{eq:ridge} may be non-convex.

\subsection{Main result}
Our main result is a regret bound for LSVI-UCB
under~\pref{assum:completeness}.
\begin{theorem}
\label{thm:main}
For any episodic MDP, with~\pref{assum:glm-regularity}
and~\pref{assum:completeness}, and for any $T$, the cumulative regret
of~\pref{alg:lsvi-ucb} is\footnote{We use $\Otilde\rbr{\cdot}$ to
  suppress factors of $M,K,\kappa,\Gamma$ and any logarithmic dependencies on
  the arguments.}
\begin{align*}
\ifthenelse{\equal{\version}{arxiv}}{
\textrm{Reg}(T) \leq \order\rbr{H\sqrt{T\ln(TH)} + HK\kappa^{-1}\sqrt{(M+K+d^2\ln(KTH))\cdot Td\ln(T/d)}} = \Otilde\rbr{H\sqrt{d^3T}},
}{
&\order\rbr{HK\kappa^{-1}\sqrt{(M+K+d^2\ln(KTH))\cdot Td\ln(T/d)}}\\
 &= \Otilde\rbr{H\sqrt{d^3T}},
}
\end{align*}
with probability $1-1/(TH)$.
\end{theorem}

The result states that LSVI-UCB enjoys $\sqrt{T}$-regret for any
episodic MDP problem and any GLM, provided that the regularity
conditions are satisfied and that optimistic closure holds. As we have
mentioned, these assumptions are relatively mild, encompassing the
tabular setting and prior work on linear function
approximation. Importantly, no explicit dynamics assumptions are
required. Thus,~\pref{thm:main} is one of the most general results we
are aware of for provably efficient exploration with function
approximation.

Nevertheless, to develop further intuition for our bound, it is worth
comparing to prior results. First, in the linear MDP setting
of~\citet{jin2019provably}, we use the identity link function so that
$K=\kappa=1$ and $M=1$, and we also are guaranteed to
satisfy~\pref{assum:completeness}. In this case, our bound differs
from that of~\citet{jin2019provably} only in the dependence on $H$,
which arises due to a difference in normalization. 
Our bound is essentially equivalent to theirs and can therefore be
seen as a strict generalization.

To capture the tabular setting, we use the standard basis
featurization as in~\pref{fact:tabular} and the identity link
function, which gives $d = |\Scal||\Acal|$, $K=\kappa=1$, and
$M=1$. Thus, we obtain the following corollary:
\begin{corollary}
For MDPs with finite state and action spaces, using feature map
$\phi(s,a) \defeq e_{s,a}\in
\RR^{|\Scal|\times|\Acal|}$, for any $T$, the cumulative regret of ~\pref{alg:lsvi-ucb} is
$\Otilde\rbr{H\sqrt{|\Scal|^3|\Acal|^3 T}}$,
with probability $1-1/(TH)$.
\end{corollary}
Note that this bound is polynomially worse than the near-optimal
$\Otilde(H\sqrt{SAT} + H^2S^2A\log(T))$ bound
of~\citet{azar2017minimax}. However,~\pref{alg:lsvi-ucb} is almost
equivalent to their algorithm, and, indeed, a refined analysis
specialized to the tabular setting can be shown to obtain a matching
regret bound. Of course, our algorithm and analysis address
significantly more complex settings than tabular MDPs, which we
believe is more important than recovering the optimal guarantee for
tabular MDPs.

\section{Proof Sketch}
\label{sec:sketch}

We now provide a brief sketch of the proof of~\pref{thm:main},
deferring the technical details to the appendix. The proof has three
main components: a regret decomposition for optimistic $Q$ learning, a
deviation analysis for least squares with GLMs to ensure optimism, and
a potential argument to obtain the final regret bound.

\paragraph{A regret decomposition.}
The first step of the proof is a regret decomposition that applies
generically to optimistic algorithms.\footnote{Related results appear
  elsewhere in the literature focusing on the tabular setting, see e.g.,~\citet{simchowitz2019non}.} The
lemma demonstrates concisely the value of optimism in reinforcement
learning, and is the primary technical motivation for our interest in
designing optimistic algorithms.

We state the lemma more generally, which requires some additional notation. Fix
round $t$ and let $\{\bar{Q}_{h,t-1}\}_{h\leq H}$ denote the current
estimated $Q$ functions. The precondition is that $\bar{Q}_{h,t-1}$ is
optimistic and has controlled overestimation. Precisely, there exists
a function $\conf_{h,t-1}: \Scal \times \Acal \to \RR_+$ such that
\begin{align}
\ifthenelse{\equal{\version}{arxiv}}{
\forall s,a,h: Q^\star_h(s,a) \leq \bar{Q}_{h,t-1}(s,a) \leq \Tcal_h(\bar{Q}_{h+1,t-1})(s,a) + \conf_{h,t-1}(s,a). \label{eq:opt_pre}
}{
Q^\star_h(s,a) &\leq \bar{Q}_{h,t-1}(s,a)\\
\bar{Q}_{h,t-1}(s,a) &\leq \Tcal_h(\bar{Q}_{h+1,t-1})(s,a) + \conf_{h,t-1}(s,a) \label{eq:opt_pre}
}
\end{align}
We now state the lemma and an immediate corollary.
\begin{lemma}
\label{lem:regret-decomposition}
Fix episode $t$ and let $\Fcal_{t-1}$ be the filtration of
$\{(s_{h,\tau},a_{h,\tau},r_{h,\tau})\}_{\tau < t}$. Assume that
$\bar{Q}_{h,t-1}$ satisfies~\pref{eq:opt_pre} for some function
$\conf_{h,t-1}$. Then, if $\pi_t = \argmax_{a\in\Acal}
\bar{Q}_{h,t-1}(\cdot,a)$ is deployed we have
\begin{align*}
V^\star - \EE\sbr{\sum_{h=1}^H r_{h,t} \mid \Fcal_{t-1}} \leq \zeta_t + \sum_{h=1}^H \conf_{h,t-1}(s_{h,t},a_{h,t}),
\end{align*}
where $\EE\sbr{\zeta_t \mid \Fcal_{t-1}} = 0$ and $\abr{\zeta_t} \leq
2H$ almost surely.
\end{lemma}
\begin{corollary}
\label{corr:regret_ub}
Assume that for all $t$, $\bar{Q}_{h,t-1}$ satisfies~\pref{eq:opt_pre}
and that $\pi_t$ is the greedy policy with respect to
$\bar{Q}_{h,t-1}$. Then with probability at least $1-\delta$, we have
\begin{align*}
\mathrm{Reg}(T) \leq \sum_{t=1}^T\sum_{h=1}^H \conf_{h,t-1}(s_{h,t},a_{h,t}) + \order(H\sqrt{T\log(1/\delta)}).
\end{align*}
\end{corollary}

The lemma states that if $\bar{Q}_{h,t-1}$ is optimistic and we deploy
the greedy policy $\pi_t$, then the per-episode regret is controlled
by the overestimation error of $\bar{Q}_{h,t-1}$, up to a stochastic
term that enjoys favorable concentration properties. Crucially, the
errors are accumulated on the observed trajectory, or, stated another
way, the $\conf_{h,t-1}$ is evaluated on the states and actions
visited during the episode. As these states and actions will be used
to update $\bar{Q}$, we can expect that the $\conf$ function will
decrease on these arguments. This can yield one of two outcomes:
either we will incur lower regret in the next episode, or we will
explore the environment by visiting new states and actions. In this
sense, the lemma demonstrates how optimism navigates the
exploration-exploitation tradeoff in the multi-step RL setting,
analogously to the bandit setting.

Note that these results do not assume any form for $\bar{Q}_{h,t-1}$
and do not require~\pref{assum:completeness}. In particular, they are
not specialized to GLMs. In our proof, we use the GLM representation
and~\pref{assum:completeness} to ensure that~\pref{eq:opt_pre} holds
and to bound the confidence sum in~\pref{corr:regret_ub}.  We believe
these technical results will be useful in designing RL algorithms for
general function classes, which is a natural direction for future
work.

\paragraph{Deviation analysis.}
The next step of the proof is to design the $\conf$ function and
ensure that~\pref{eq:opt_pre} holds, with high probability. This is
the contents of the next lemma.

\begin{lemma}
\label{lem:ucb}
Under~\pref{assum:glm-regularity} and~\pref{assum:completeness}, with
probability $1-1/(TH)$, we have that
\begin{align*}
\forall t,h,s,a: \abr{f(\langle \phi(s,a),\hat{\theta}_{h,t}\rangle) - \Tcal_h(\bar{Q}_{h+1,t})(s,a)} \leq \min\cbr{2,\gamma \nbr{\phi(s,a)}_{\Lambda_{h,t}^{-1}}},
\end{align*}
where $\gamma, \Lambda_{h,t}$ are defined in~\pref{alg:lsvi-ucb}.
\end{lemma}

A simple induction argument then verifies that~\pref{eq:opt_pre}
holds, which we summarize in the next corollary.
\begin{corollary}
\label{corr:conf}
Under~\pref{assum:glm-regularity} and~\pref{assum:completeness}, with
probability $1-1/(TH)$, we have that~\pref{eq:opt_pre} holds for all
$t,h$ with $\conf_{h,t-1}(s,a) =
\gamma\nbr{\phi(s,a)}_{\Lambda_{h,t-1}^{-1}}$.
\end{corollary}

The proof of the lemma requires an intricate deviation analysis to
account for the dependency structure in the data sequence. The
intuition is that, thanks to~\pref{assum:completeness} and the fact
that $\bar{Q}_{h+1,t} \in \Gu$, we know that there exists a parameter
$\bar{\theta}_{h,t}$ such that $f(\langle \phi(s,a),
\bar{\theta}_{h,t}\rangle) = \Tcal_h(\bar{Q}_{h+1,t})(s,a)$. It is
easy to verify that $\bar{\theta}_{h,t}$ is the Bayes optimal
predictor for the square loss problem in~\pref{eq:ridge}, and so with
a uniform convergence argument we can expect that $\hat{\theta}_{h,t}$
is close to $\bar{\theta}_{h,t}$, which is our desired conclusion.

There are two subtleties with this argument. First, we want to show
that $\bar{\theta}_{h,t}$ and $\hat{\theta}_{h,t}$ are close in a
data-dependent sense, to obtain the dependence on the
$\Lambda_{h,t}^{-1}$-Mahalanobis norm in the bound. This can be done
using vector-valued self-normalized martingale
inequalities~\citep{pena2008self}, as in prior work on linear stochastic
bandits~\citep{abbasi2012online,filippi2010parametric,abbasi2011improved}.

However, the process we are considering is not a martingale, since
$\bar{Q}_{h+1,t}$, which determines the regression targets
$y_{h,\tau}$, depends on all data collected so far. Hence $y_{h,\tau}$
is not measurable with respect to the filtration $\Fcal_\tau$, which
prevents us from directly applying a self-normalized martingale
concentration inequality.  To circumvent this issue, we use a uniform
convergence argument and introduce a deterministic covering of
$\Gu$. Each element of the cover induces a different sequence of
regression targets $y_{h,\tau}$, but as the covering is deterministic,
we do obtain martingale structure. Then, we show that the error term
for the random $\bar{Q}_{h+1,t}$ that we need to bound is close to a
corresponding term for one of the covering elements, and we finish the
proof with a uniform convergence argument over all covering elements.


The corollary is then obtained by a straightforward inductive
argument. Assuming $\bar{Q}_{h+1,t}$ dominates $Q^\star$, it is easy
to show that $\bar{Q}_{h,t}$ also dominates $Q^\star$, and the upper
bound is immediate. Combining~\pref{corr:conf}
with~\pref{corr:regret_ub}, all that remains is to upper bound the
confidence sum.

\paragraph{The potential argument.}
To bound the confidence sum, we use a relatively standard potential
argument that appears in a number of works on stochastic (generalized)
linear bandits. We summarize the conclusion with the following lemma,
which follows directly from Lemma 11 of~\citet{abbasi2012online}.

\begin{lemma}
\label{lem:elliptic}
For any $h \leq H$ we have that $\sum_{t=1}^T
\nbr{\phi(s_{h,t},a_{h,t})}_{\Lambda_{h,t-1}^{-1}}^2 \leq
2d\ln(1+T/d)$.
\end{lemma}

\paragraph{Wrapping up.}
To prove~\pref{thm:main} we first note that via~\pref{lem:elliptic}
and an application of the Cauchy-Schwarz inequality, we have that for
each $h\leq H$
\begin{align*}
\sum_{t=1}^T \conf_{h,t-1}(s_{h,t},a_{h,t}) \leq \gamma\sqrt{T}\sqrt{\sum_{t=1}^T\nbr{\phi(s_{h,t},a_{h,t})}_{\Lambda_{h,t-1}^{-1}}^2} \leq \order\rbr{\gamma\sqrt{Td\ln(T/d)}}
\end{align*}
Invoking~\pref{corr:conf},~\pref{corr:regret_ub}, and the
definition of $\gamma$ yields the $\Otilde\rbr{H\sqrt{d^3T}}$
regret bound.

\section{Discussion}
This paper presents a provably efficient reinforcement learning
algorithm that approximates the $Q^\star$ function with a generalized
linear model. We prove that the algorithm obtains
$\Otilde(H\sqrt{d^3T})$ regret under mild regularity conditions and a
new expressivity condition that we call \emph{optimistic
  closure}. These assumptions generalize both the tabular setting,
which is classical, and the linear MDP setting studied in recent
work. Further they represent the first statistically and
computationally efficient algorithms for reinforcement learning with
generalized linear function approximation, without explicit dynamics
assumptions.

We close with some open problems. First, using the fact
that~\pref{corr:regret_ub} applies beyond GLMs, can we develop
algorithms that can employ general function classes? While such
algorithms do exist for the contextual bandit
setting~\citep{foster2018practical}, it seems quite difficult to
generalize this analysis to multi-step reinforcement learning. More
importantly, while optimistic closure is weaker than some prior
assumptions (and incomparable to others), it is still quite strong,
and stronger than what is required for the batch RL setting. An
important direction is to investigate weaker assumptions that enable
provably efficient reinforcement learning with function
approximation. We look forward to studying these questions in future
work.

\section*{Acknowledgements}
We thank Wen Sun for helpful conversations during the development of this paper.
Simon S. Du is supported by National Science Foundation (Grant
No. DMS-1638352) and the Infosys Membership.

\clearpage

\appendix
\section{Proofs of basic results}
\label{app:simple}

\begin{proof}[Proof of~\pref{fact:realizability}]
We will solve for $Q^\star$ via dynamic programming, starting from
time point $H$. In this case, the Bellman update operator is
degenerate, and we start by observing that $\Tcal_H(g) \equiv
Q^\star_{H}$ for all $g$. Consequently we have $Q^\star_H \in \Gcal$.
Next, inductively we assume that we have $Q^\star_{h+1} \in \Gcal$,
which implies that $Q^\star_{h+1} \in \Gu$ as we may take the same
parameter $\theta$ and set $A \equiv 0$. Then, by the standard Bellman
fixed-point characterization, we know that $Q^\star_h =
\Tcal_h(Q^\star_{h+1})$, at which point~\pref{assum:completeness}
yields that $Q^\star_h \in \Gcal$.
\end{proof}

\begin{proof}[Proof of~\pref{fact:tabular}]
We simply verify that $\Gcal$ contains \emph{all} mappings from $(s,a)
\mapsto [0,1]$, at which point the result is immediate. To see why,
observe that via~\pref{assum:glm-regularity} we know that $f$ is
invertible (it is monotonic with derivative bounded from above and
below). Then, note that any function $(s,a)\mapsto [0,1]$ can be
written as a vector $v \in [0,1]^{|\Scal|\times|\Acal|}$. For such a
vector $v$, if we define $\theta_{s,a} \defeq f^{-1}(v_{s,a})$ we have
that $f(\langle e_{s,a},\theta\rangle) = v_{s,a}$. Hence $\Gcal$
contains all functions, so we trivially have optimistic closure.
\end{proof}

\section{Proof of~\pref{thm:main}}
To facilitate our regret analysis we define the following important
intermediate quantity:
\begin{align*}
\bar{\theta}_{h,t} \in \BB_d: ~~ f(\langle\phi(s,a),\bar\theta_{h,t}\rangle) \defeq  \EE\sbr{r_h + \max_{a'\in\Acal} \bar{Q}_{h+1,t}(s',a') \mid s,a}.
\end{align*}
In words, $\bar{\theta}_{h,t}$ is the Bayes optimal predictor for the
squared loss problem at time point $h$ in the $t^{\textrm{th}}$
episode. Since by inspection $\bar{Q}_{h+1,t} \in \Gu$,
by~\pref{assum:completeness} we know that $\bar{\theta}_{h,t}$
exists for all $h$ and $t$.

\begin{lemma}
\label{lem:delta}
For any $\theta,\theta',x\in\mathbb R^d$ satisfying $\|\theta\|_2,\|\theta'\|_2,\|x\|_2\leq 1$, 
\begin{align*}
\kappa^2\abr{\langle x,\theta'-\theta\rangle}^2 \leq \abr{f(\inner{x}{\theta'})-f(\inner{x}{\theta})}^2\leq K^2\nbr{\theta'-\theta}_2^2.
\end{align*}
\end{lemma}
\begin{proof}
By the mean-value theorem, there exists $\tilde{\theta} =
\theta+\lambda(\theta'-\theta)$ for some $\lambda\in(0,1)$ such that
$f(\inner{x}{\theta'})-f(\inner{x}{\theta}) = \inner{\nabla_\theta
  f(\langle x,\tilde\theta\rangle)}{\theta'-\theta}$.  On the other
hand, by the chain rule and~\pref{assum:glm-regularity},
$\nabla_\theta f(\langle x,\tilde\theta\rangle)= f'(\langle x, \tilde\theta\rangle)\cdot x$.
Hence, 
\begin{align*}
| \langle\nabla_\theta f(x^\top\tilde\theta),\theta'-\theta\rangle|^2 &\leq f'(\langle x,\tilde{\theta}\rangle)^2  \cdot  \abr{\inner{x}{\theta' - \theta}}^2 \leq K^2 \nbr{x}_2^2 \nbr{\theta' - \theta}_2^2 \leq K^2\nbr{\theta' - \theta}_2^2;\\
| \langle\nabla_\theta f(x^\top\tilde\theta),\theta'-\theta\rangle|^2 &\geq \kappa^2 \abr{\inner{x}{\theta' - \theta}}^2,
\end{align*}
which are to be demonstrated.
\end{proof}

\begin{lemma}
\label{lem:G-covering}
For any $0 < \veps \leq 1$, there exists a finite subset
$\Vcal_{\veps} \subset \Gu$ with $\ln|\Vcal_\veps| \leq
6d^2\ln(2(1+K+\Gamma)/\veps)$, such that
\begin{equation}
\sup_{g\in \Gu}\min_{v \in \Vcal_\veps} \sup_{s,a}\abr{g(\phi(s,a)) - v(\phi(s,a))} \leq \veps.
\label{eq:G-covering}
\end{equation}
\end{lemma}
\begin{proof}
Recall that for every $g\in\Gu$, there exists $\theta \in \BB_d$, $0
\leq \gamma\leq\Gamma$ and $\|A\|_{\mathrm{op}}\leq 1$ such that
$g(x)=\min\{1,f(\inner{x}{\theta}) + \gamma\nbr{x}_A\}$.  Let
$\Theta_\varepsilon \subseteq \BB_d$, $\Gamma_{\varepsilon} \subseteq
     [0,\Gamma]$ and $\Mcal_{\varepsilon} \subseteq \{ M \in
     \mathbb{S}_d^+: \nbr{M}_{\mathrm{op}} \leq 1\}$ be finite subsets
     such that for any $\theta,\gamma,A$, there exist $\theta'
     \in\Theta_{\veps}$, $\gamma' \in \Gamma_{\veps}$, $A' \in
     \Mcal_{\veps}$ such that
\begin{align*}
\max\cbr{ \nbr{\theta-\theta'}_2, \abr{\gamma-\gamma'}, \nbr{A-A'}_{\mathrm{op}}} \leq \veps',
\end{align*}
where $\veps'\in (0,1)$ will be specified later in the proof.  For the
function $g \in \Gu$ corresponding to the parameters $\theta,\gamma,A$
the function $g'$ corresponding to parameters $\theta',\gamma',A'$ satisfies
\begin{align*}
\sup_{s,a}\abr{g(\phi(s,a))-g'(\phi(s,a))} &\leq \sup_{x \in \BB_d} \abr{g(x)-g'(x)}\\
& \leq \sup_{x \in \BB_d}\abr{f(\inner{x}{\theta})-f(\inner{x}{\theta'})+ \gamma\nbr{x}_A-\gamma' \nbr{x}_{A'}}\\
&\leq K\nbr{\theta-\theta'}_2 + \abr{\gamma-\gamma'} + \Gamma\abr{ \nbr{x}_A- \nbr{x}_{A'}}\\
&\leq K\nbr{\theta-\theta'}_2 + \abr{\gamma-\gamma'} + \Gamma\sqrt{\abr{x^\top (A-A') x}}\\
&\leq K\veps' + \veps' + \Gamma\sqrt{\veps'} \leq (1+K+\Gamma)\sqrt{\epsilon'}.
\end{align*}
In the last step we use $\veps' \leq 1$.  Therefore,
if we define the class $\Vcal_{\veps} \defeq \{(s,a) \mapsto
\min\{1,f(\inner{\phi(s,a)}{\theta'}) + \gamma'\nbr{\phi(s,a)}_{A'} :
\theta' \in \Theta_{\veps},\gamma \in \Gamma_{\veps},A \in
\Mcal_{\veps}\}$, we know that the covering property is satisfied with
parameter $(1+K+\Gamma)\sqrt{\veps'}$. Setting $\veps' =
\veps^2/(1+K+\Gamma)^2$ we have the desired covering property.


Finally, we upper bound $\ln|\Vcal_{\veps}|$.  By definition, we have
that $\ln|\Vcal_{\veps}| \leq \ln|\Theta_{\veps}| +
\ln|\Gamma_{\veps}|+\ln|\Mcal_{\veps}|$. Furthermore, standard
covering number bounds reveals that $\ln|\Theta_\veps|\leq
d\ln(2/\veps')$, $\ln|\Gamma_\veps|\leq \ln(1/\veps')$ and
$\ln|\Mcal_{\veps}|\leq d^2\ln(2/\veps')$. 
Plugging in the definition of $\veps'$ yields the result.
\end{proof}

For the next lemma, let $\Fcal_{t-1} \defeq
\sigma(\{(s_{h,\tau},a_{h,\tau},r_{h,\tau})\}_{\tau<t})$ be the
filtration induced by all observed trajectories up to but not
including time $t$. Observe that $\bar{Q}_{\cdot,t-1}$ and our policy
$\hat{\pi}_{h,t}$ are $\Fcal_{t-1}$ measurable.

\begin{lemma}
\label{lem:glm-error}
Fix any $1 \leq t \leq T$ and $1 \leq h \leq H$. 
Then as long as $\pi_t$ is $\Fcal_{t-1}$ measurable, with probability
$1-1/(TH)^2$ it holds that
\begin{align*}
\abr{ f(\langle \phi(s,a), \hat{\theta}_{h,t}\rangle) - f(\langle \phi(s,a), \bar{\theta}_{h,t}\rangle)} \leq \min\cbr{2, \gamma \nbr{\phi(s,a)}_{\Lambda_{h,t}^{-1}}}, \;\;\;\;\;\forall s,a.
\end{align*}
for $\gamma \geq CK\kappa^{-1}\sqrt{1+M+K+d^2\ln((1+K+\Gamma)TH)}$ and $0 < C< \infty$ is a universal constant. 
\end{lemma}
Note that this is precisely~\pref{lem:ucb}, as $\bar{\theta}_{h,t}$ is
defined as $f(\langle \phi(s,a),\bar{\theta}_{h,t}) =
\Tcal_h(\bar{Q}_{h+1,t})(s,a)$.
\begin{proof}
The upper bound of $2$ is obvious, since both terms are upper bounded
by $1$ in absolute value. 
 Therefore we focus on the second term in
the minimum.  To simplify notation we omit the dependence on $h$ in
the subscripts and write $x_\tau,y_\tau$ for $x_{h,\tau}$ and
$y_{h,\tau}$.  We also abbreviate $\hat{\theta} \defeq
\hat{\theta}_{h,t}$ and $\bar{\theta} \defeq \bar{\theta}_{h,t}$.

Since $\nbr{\bar\theta}_2\leq 1$, the optimality of $\hat{\theta}$ for~\pref{eq:ridge} implies that
\begin{align*}
\sum_{\tau\leq t}\rbr{f(\langle x_\tau,\hat{\theta}\rangle) - y_\tau}^2  \leq \sum_{\tau\leq t}\rbr{f(\langle x_\tau,\bar\theta\rangle) - y_\tau}^2 .
\end{align*}
Decomposing the squares and re-organizing the terms, we have that
\begin{align}
\sum_{\tau\leq t}\rbr{f(\langle x_\tau,\hat\theta\rangle) - f(\langle x_\tau,\bar\theta\rangle)}^2 \leq 2\abr{ \sum_{\tau\leq t}\xi_{\tau} (f(\langle x_\tau ,\hat{\theta}\rangle) - f(\langle x_\tau,\bar{\theta}\rangle))},\label{eq:basic}
\end{align}
where $\xi_\tau \defeq y_\tau-f(\langle x_\tau,\bar\theta\rangle)$.
By the fundamental theorem of calculus, we have
\begin{align*}
f(\langle x_\tau,\hat{\theta}\rangle ) - f(\langle x_\tau,\bar{\theta}\rangle) = \int_{\langle x_\tau,\bar{\theta}\rangle}^{\langle x_\tau, \hat{\theta}\rangle} f(s) \ud s = \langle x_\tau, \hat{\theta}-\bar{\theta}\rangle \underbrace{\int_{0}^1 f'(\langle x_\tau, s\hat{\theta} - (1-s)\bar{\theta}\rangle) \ud s}_{\defeq D_\tau}.
\end{align*}
Using this identity on both sides of~\pref{eq:basic}, we have that
\begin{align}
\sum_{\tau\leq t}D_\tau^2\rbr{\langle x_\tau,\hat{\theta}-\bar{\theta}\rangle}^2 \leq 2\abr{\sum_{\tau\leq t}\xi_\tau D_\tau \langle x_\tau, \hat{\theta}-\bar{\theta}\rangle}.
\label{eq:glm-1}
\end{align}
Note also that, by~\pref{assum:glm-regularity}, $D_\tau$ satisfies $\kappa^2\leq D_\tau^2\leq K^2$ almost surely for all $\tau$.

The difficulty in controlling~\pref{eq:glm-1} is that $\bar{\theta}$
itself is a random variable that depends on $\{(x_\tau,y_\tau)\}_{\tau
  \leq t}$. In particular, we want that $\EE[\xi_\tau \mid D_\tau
  \inner{x_\tau}{\phi},\Fcal_{\tau-1}] = 0$ for any fixed $\phi$, but
this is not immediate as $\bar{\theta}$ depends on $x_\tau$.  To
proceed, we eliminate this dependence with a uniform convergence
argument.  Let $\veps \in (0,1)$ be a covering accuracy parameter to
be determined later in this proof.  Let $\Vcal_\veps$ be the pointwise
covering for $\Gu$ that is implied by~\pref{lem:G-covering}.  Let
$g_\veps \in \Vcal_\veps$ be the approximation for $\bar{Q}_{h+1,t}$
that satisfies~\pref{eq:G-covering}. By~\pref{assum:completeness},
there exists some $\theta^\sharp \in \BB_d$ such that
\begin{align*}
\forall s,a: ~~ f(\langle \phi(s,a), \theta^\sharp\rangle) = \EE\sbr{ r + \max_{a' \in \Acal} g_\veps(s',a') \mid s,a}.
\end{align*}
Now, define $y_\tau^\sharp$ and $\xi_\tau^\sharp$ as
\begin{align*}
y_\tau^\sharp \defeq r_{h,\tau} + \max_{a'\in \Acal}g_{\veps}(s_{h+1,\tau},a'), \qquad 
\xi_\tau^\sharp \defeq y_\tau^\sharp - f(\langle x_{h,\tau}, \theta^\sharp\rangle).
\end{align*}
The right-hand side of~\pref{eq:glm-1} can then be upper bounded as
\begin{equation}
2\abr{\sum_{\tau \leq t} \xi_\tau D_\tau \langle x_\tau, \hat{\theta} - \bar{\theta}\rangle} \leq 2\abr{\sum_{\tau\leq t}\xi_\tau^\sharp D_\tau \langle x_\tau,\hat\theta-\bar\theta\rangle} + \Delta,
\label{eq:glm-1.5}
\end{equation}
where $|\Delta|\leq Kt\times \max_{\tau\leq
  t}|\xi_\tau^\sharp-\xi_\tau|$ almost surely.

\paragraph{Upper bounding $\Delta$ in~\pref{eq:glm-1.5}.}
 Fix $\tau\leq t$. By definition, we have that
\begin{align}
\abr{\xi_\tau^\sharp-\xi_\tau}
&\leq \abr{y_\tau^\sharp - y_\tau} + \abr{ f(\langle x_\tau,\bar\theta\rangle) - f(\langle x_\tau,\theta^\sharp\rangle)}\nonumber\\
&\leq \max_{a\in\Acal}\abr{g_{\veps}(s_{h+1,\tau},a)-\bar{Q}_{h+1,t}(s_{h+1,\tau},a)} + K\nbr{\bar{\theta}-\theta^\sharp}_2\label{eq:glm-1.6}\\
&\leq \epsilon + K\epsilon\leq (K+1)\epsilon,\label{eq:glm-1.7}
\end{align}
where~\pref{eq:glm-1.6} holds by~\pref{lem:delta}
and~\pref{eq:glm-1.7} follows from~\pref{lem:G-covering}. In
particular, the bound on $\nbr{\bar{\theta} - \theta^\sharp}_2$ can be
verified by expanding the definitions and noting that $g_\veps$ is
pointwise close to $\bar{Q}_{h+1,t}$. Therefore, we have
\begin{align}
\abr{\Delta} \leq (K+1)^2t\epsilon.\label{eq:Delta-upper-bound}
\end{align}

\paragraph{Upper bounding~\pref{eq:glm-1.5}.}
Note that $D_\tau$ is a function of $x_\tau$, $\hat\theta$, and $\bar\theta$.
For clarity, we define $D_\tau(\theta,\theta') := \int_0^1 f'(\langle x_\tau,s\theta+(1-s)\theta')\rangle)\ud s$.
As $|f''(z)|\leq M$ for all $|z|\leq 1$ and $\nbr{x_\tau}_2\leq 1$, we have that for every $\theta,\theta',\tilde{\theta},\tilde{\theta}' \in \BB_d$
\begin{align*}
\abr{D_\tau(\theta,\theta')-D_\tau(\tilde\theta,\tilde\theta')}
&\leq \int_0^1\abr{f'(\langle x_\tau,s\theta+(1-s)\theta'\rangle) - f'(\langle x_\tau,s\tilde{\theta}+(1-s)\tilde{\theta}'\rangle)}\ud s\\
&\leq M(\|\theta-\tilde\theta\|_2 + \|\theta'-\tilde\theta'\|_2).
\end{align*}
Hence, for any $(\theta,\theta')$ and $(\tilde\theta,\tilde\theta')$ pairs, we have for every $\tau$ that  
\begin{align*}
& \abr{\xi_\tau^\sharp \inner{x_\tau}{ D_\tau(\theta,\theta')(\theta-\theta') - D_\tau(\tilde\theta,\tilde\theta')(\tilde\theta-\tilde\theta')}}\\
&
\leq \big|D_\tau(\theta,\theta')-D_\tau(\tilde\theta,\tilde\theta')\big|\times \|\theta-\theta'\|_2 + \big|D_\tau(\tilde\theta,\tilde\theta')\big|\times (\|\theta-\tilde\theta\|_2 + \|\theta'-\tilde\theta'\|_2)\\
&\leq M(\|\theta-\tilde\theta\|_2 + \|\theta'-\tilde\theta'\|_2)\times 2 + K(\|\theta-\tilde\theta\|_2 + \|\theta'-\tilde\theta'\|_2)\\
&\leq (2M+K)(\|\theta-\tilde\theta\|_2 + \|\theta'-\tilde\theta'\|_2).
\end{align*}
Here we are using that $\abr{\xi_\tau} \leq 1$.  

We are now in a position to
invoke~\pref{lem:self-normalized}. Consider a fixed function
$g_\veps$, which defines a fixed $\theta^\sharp$. We will bound
$\abr{\sum_{\tau \leq t}\xi_\tau^\sharp\langle x_\tau,
  D_\tau(\theta,\theta')(\theta - \theta')\rangle}$ uniformly over all
pairs $(\theta,\theta')$. With $g_\veps,\theta^\sharp$ fixed and since
$\pi_t$ is $\Fcal_{t-1}$ measurable, we have that
$\{x_\tau,\xi_\tau^\sharp\}_{\tau \leq t}$ are random variables
satisfying $\EE[\xi_\tau^\sharp \mid x_{1:\tau},\xi_{1:\tau-1}^\sharp]
= 0$. For $\phi= (\theta,\theta')$ we define the function
$q(x_\tau,\phi) = \langle x, D_\tau(\phi) (\theta-\theta')\rangle$,
which as we have just calculated satisfies $\abr{q(x_\tau,\phi) -
  q(x_\tau,\phi')} \leq (2M+K)\nbr{\phi-\phi'}_2$.  For $\delta' \in
(0,1/2)$ with probability $1-\delta'$ we have $\forall \phi =
(\theta,\theta') \in \BB_d^2$:
\begin{align}
& \abr{\sum_{\tau\leq t} \xi_\tau^\sharp\langle x_\tau, D_\tau(\phi)(\theta - \theta')\rangle} \leq (2M+K) + 2\rbr{1 + \sqrt{V(\phi)}}\sqrt{2d\ln(4T)+\ln(1/\delta')}\notag\\
& ~~~~~~~ \leq 4\max\cbr{M+K+\sqrt{2d\ln(4T)+\ln(1/\delta')}, \sqrt{V(\phi)}\sqrt{2d\ln(4T)+\ln(1/\delta')}},\label{eq:glm-2}
\end{align}
where $V(\phi) \defeq \sum_{\tau\leq t}\langle
x_\tau,D_\tau(\phi)(\theta - \theta')\rangle^2$. The last inequality
holds because $a+b \leq 2\max\{a,b\}$.

Next, take a union bound over all $g_\veps \in \Vcal_{\veps}$
so~\pref{eq:glm-2} holds for any $g_\veps$ and any subsequently
induced choice of $\xi_\tau^\sharp$ with probability at least
$1-|\Vcal_\veps|\delta'$. In particular, this union bound implies
that~\pref{eq:glm-2} holds for the choice of $g_\veps$ that
approximates $\bar{Q}_{h+1,t}$. Therefore,
combining~\pref{eq:glm-1},~\pref{eq:glm-1.5},~\pref{eq:Delta-upper-bound}
with~\pref{eq:glm-2} for this choice of $g_\veps$, we have that with
probability at least $1-|\Vcal_\veps|\delta'$
\begin{align*}
&\sum_{\tau\leq t} D_\tau^2 \langle x_\tau,\hat{\theta} - \bar{\theta}\rangle^2 \leq 2\Delta + 2\abr{\sum_{\tau\leq t} \xi_\tau^\sharp \langle x_\tau, D_\tau (\hat{\theta} - \bar{\theta})\rangle}\\
& ~~~~ \leq 2(K+1)^2t\veps + 8\max\cbr{M+K+\sqrt{2d\ln(4T)+\ln(|\Vcal_\veps|/\delta')}, \sqrt{V(\hat{\theta},\bar{\theta})}\cdot\sqrt{2d\ln(4T)+\ln(|\Vcal_\veps|/\delta')}}.
\end{align*}
Observe that the left hand side is precisely
$V(\hat{\theta},\bar{\theta})$. Now, set $\veps = 1/(2(K+1)^2T)$ and
$\delta' = 1/(|\Vcal_{\veps}|T^2H^2)$ and use the bound on
$\ln|\Vcal_{\veps}|$ from~\pref{lem:G-covering} to get 
\begin{align*}
& \sqrt{2d\ln(4T)+\ln(|\Vcal_\veps|/\delta')} \leq \sqrt{2d\ln(4T) + 12d^2\ln(2(1+K+\Gamma)/\veps) + 2\ln(TH)} \\
& ~~~~~~ \leq \sqrt{4d\ln(2TH) + 24d^2\ln(2(1+K+\Gamma)T)} \leq \sqrt{28d^2\ln(2(1+K+\Gamma)TH)}
\end{align*}
Therefore, we obtain
\begin{align*}
V(\hat{\theta},\bar{\theta}) 
& \leq 1 + 8\max\cbr{ M+K+\sqrt{28d^2\ln(2(1+K+\Gamma)TH)}, \sqrt{V(\hat{\theta},\bar{\theta})}\cdot\sqrt{28d^2\ln(2(1+K+\Gamma)TH)}}\\
& \leq 16\max\cbr{ 1+M+K+\sqrt{28d^2\ln(2(1+K+\Gamma)TH)}, \sqrt{V(\hat{\theta},\bar{\theta})}\cdot\sqrt{28d^2\ln(2(1+K+\Gamma)TH)}}.
\end{align*}

Subsequently, 
\begin{align*}
&V(\hat{\theta},\bar{\theta}) = \sum_{\tau\leq t}D_\tau^2 \langle x_\tau,\hat{\theta} - \bar{\theta}\rangle^2 \\
& ~~~~ \leq 16 \max\cbr{1 + M+K+\sqrt{28d^2\ln(2(1+K+\Gamma)TH)}, 448 d^2\ln(2(1+K+\Gamma)TH)}\\
& \leq C_V^2 (1+M+K + d^2\ln((1+K+\Gamma)TH)),
\end{align*}
where $0<C_V<\infty$ is a universal constant. 

Next, note that $D_\tau^2 \geq\kappa^2$, thanks to~\pref{assum:glm-regularity}. We then have
\begin{align*}
\sqrt{(\hat\theta-\bar\theta)^\top\Lambda_{h,t}(\hat\theta-\bar\theta)} \leq 
\kappa^{-1}\sqrt{V(\hat{\theta},\bar{\theta})} \leq C_V\kappa^{-1}\sqrt{1+M+K+d^2\ln((1+K+\Gamma)TH)},
\end{align*}
where $\Lambda_{h,t}=\sum_{\tau<t}x_\tau,x_\tau^\top$.  Finally, for
any $(s,a)$ pair, invoking~\pref{lem:delta} and the Cauchy-Schwarz
inequality we have
\begin{align*}
& \abr{f(\langle\phi(s,a),\hat{\theta}\rangle)-f(\langle\phi(s,a),\bar{\theta}\rangle)}
\leq K\abr{\langle\phi(s,a),\hat{\theta}-\bar{\theta}\rangle}\\
& \leq K\sqrt{(\hat{\theta}-\bar{\theta})^\top\Lambda_{h,t}(\hat{\theta}-\bar{\theta})}\times\sqrt{\phi(s,a)^\top\Lambda_{h,t}^{-1}\phi(s,a)}\\
& \leq C_VK\kappa^{-1}\sqrt{1+M+K+d^2\ln((1+K+\Gamma)TH)}\times \nbr{\phi(s,a)}_{\Lambda_{h,t}^{-1}}
\end{align*}
which is to be demonstrated.
\end{proof}

\begin{corollary}
With probability $1-1/(TH)$, $\bar{Q}_{h,t}(s,a)\geq Q_h^\star(s,a)$ holds for all $h,t,s,a$.
\label{corr:ucb}
\end{corollary}
\begin{proof}
Fix $1 \leq t \leq T$. We use induction on $h$ to prove this
corollary.  For $h=H+1$, $\bar{Q}_{H+1,t}(\cdot,\cdot)\geq
Q_{H+1}^\star(\cdot,\cdot)$ clearly holds because $\bar{Q}_{H+1,t}
\equiv Q_{H+1}^\star \equiv 0$.  Now assume that $\bar{Q}_{h+1,t} \geq
Q_{h+1}^\star$, and let us prove that this is also true for time step
$h$.

Since $\bar{Q}_{h+1,t}(s',a')\geq Q_{h+1}^\star(s',a')$ for all
$s',a'$, we have that $f(\langle\phi(s,a),\bar{\theta}_{h,t}\rangle)
\geq f(\langle\phi(s,a),\theta_h^\star\rangle)$ for all $(s,a)$
pairs. Then, by the definition of $\bar{Q}_{h,t}$
and~\pref{lem:glm-error}, with probability $1-1/(TH)^2$ it holds
uniformly for all $(s,a)$ pairs that $\bar{Q}_{h,t}(s,a) \geq
f(\langle\phi(s,a),\bar{\theta}_{h,t}\rangle)$.  Hence, with the same
probability, we have $\bar{Q}_{h,t}(s,a)\geq Q_h^\star(s,a)$ for all
$(s,a)$ pairs.  A union bound over all $t\leq T$ and $h\leq H$
completes the proof.
\end{proof}

\begin{lemma}[Restatement of~\pref{lem:regret-decomposition}]
Fix $t\leq T$ and let $\mathcal F_{t-1}$ be the filtration of
$\{(s_{h,\tau},a_{h,\tau},r_{h,\tau})\}_{\tau<t}$. Assume that
$\bar{Q}_{h,t-1}$ satisfies
\begin{align*}
\forall s,a,h: Q^\star_h(s,a) \leq \bar{Q}_{h,t-1}(s,a)
\leq \Tcal_h(\bar{Q}_{h+1,t-1})(s,a) + \conf_{h,t-1}(s,a),
\end{align*}
where $\conf_{h,t-1}$ is some $\Fcal_{t-1}$-measurable function. Then
we have the difference between expected total
\begin{align*}
V^\star - \EE\sbr{\sum_{h=1}^H r_{h,t} \mid \Fcal_{t-1}} \leq \zeta_t + \sum_{h=1}^H \conf_{h,t-1}(s_{h,t},a_{h,t})
\end{align*}
where $\EE[\zeta_t|\Fcal_{t-1}]=0$ and $|\zeta_t|\leq 2H$ almost surely.
\label{lem:value-decomposition}
\end{lemma}
\begin{proof}
Observe that
\begin{align*}
V^\star &= \EE\sbr{Q^\star(s_1,\pi^\star(s_1))} \leq \EE\sbr{\bar{Q}_{1,t-1}(s_1,\pi^\star(s_1))} \leq \EE\sbr{\bar{Q}_{1,t-1}(s_1,\pi_t(s_1))}\\
& \leq \EE\sbr{\conf_{1,t-1}(s_1,\pi_t(s_1))} + \EE\sbr{\Tcal_1(\bar{Q}_{2,t-1})(s_1,\pi_t(s_1))}\\
& = \EE\sbr{\conf_{1,t-1}(s_1,\pi_t(s_1))} + \EE\sbr{r_1 \mid s_1,a_1 =\pi_t(s_1)} + \EE_{s_2 \sim \pi_t}\sbr{\bar{Q}_{2,t-1}(s_2,\pi_t(s_2))}
\end{align*}
Throughout this calculation, $s_1 \sim \mu$. The first step here is by
definition, the second uses the optimism property for
$\bar{Q}_{1,t-1}$. The third uses that $\pi_t$ is the greedy policy
with respect to $\bar{Q}_{1,t-1}$ while the fourth uses the upper
bound on $\bar{Q}_{1,t-1}$. Finally we use the definition of the
Bellman operator and the fact that $\pi_t$ is the greedy policy yet
again. Comparing this upper bound with the expected reward collected
by $\pi_t$ we observe that $r_1$ cancels, and we get
\begin{align*}
V^\star - \EE\sbr{\sum_{h=1}^H r_{h,t} \mid \Fcal_{t-1}} \leq \EE_{\pi_t}\sbr{\conf_{1,t-1}(s_1,\pi_t(s_1))} + \EE_{\pi_t}\sbr{\bar{Q}_{2,t-1}(s_2,\pi_t(s_2)) - \sum_{h=2}^Hr_{h,t} \mid \Fcal_{t-1}}.
\end{align*}
At this point, notice that $\bar{Q}_{2,t-1}(s_2,\pi_t(s_2))$ is
precisely what we alreacy upper bounded at time point $h=1$ and we are
always considering the state-action distribution induced by
$\pi_t$. Hence, repeating the argument for all $h$, we obtain
\begin{align*}
V^\star - \EE\sbr{\sum_{h=1}^H r_{h,t} \mid \Fcal_{t-1}} \leq \sum_{h=1}^H \EE_{\pi_t}\sbr{\conf_{h,t-1}(s_h,\pi_t(s_h))} = \sum_{h=1}^H \conf_{h,t-1}(s_{h,t},a_{h,t}) + \zeta_t,
\end{align*}
where
\begin{align*}
\zeta_t \defeq \sum_{h=1}^H \EE_{\pi_t}\sbr{\conf_{h,t-1}(s_h,\pi_t(s_h))} - \conf_{h,t-1}(s_{h,t},a_{h,t}),
\end{align*}
which is easily seen to have the required properties. 
\end{proof}

\begin{corollary}
For any $h\leq H$, $\sum_{t\leq
  T}\nbr{\phi(s_{h,t},a_{h,t})}^2_{\Lambda_{h,t-1}^{-1}} \leq 2d \ln\rbr{1+ T/d}$.
\label{corr:elliptical-potential}
\end{corollary}
\begin{proof}
The result follows directly from Lemma 11 of~\citet{abbasi2012online},
using the fact that $\Lambda_0=I$ and $\phi(s,a) \in \BB_d$ so that
$\nbr{\phi(s_{h,t},a_{h,t})}_{\Lambda_{h,t-1}^{-1}} \leq 1$ always.
\end{proof}

\begin{theorem}
The cumulative regret of Algorithm \ref{alg:lsvi-ucb} is upper bounded by 
\begin{align*}
\widetilde{O}\rbr{H\sqrt{T} + H^2\sqrt{d^3T}},
\end{align*}
with probability at least $1-1/(TH)$. 
\label{thm:main-upper}
\end{theorem}
\begin{proof}
Assume that~\pref{corr:ucb} holds for all $1 \leq h \leq H$ and $1
\leq t \leq T$. Applying~\pref{lem:value-decomposition} and the
definition of $\conf_{h,t-1}$ implied by~\pref{corr:ucb}, the
cumulative expected regret is at most
\begin{align*}
  TV^\star - \EE\sbr{\sum_{t=1}^T\sum_{h=1}^H r_{h,t}} &\leq \sum_{t=1}^T \zeta_t + \sum_{t=1}^T \sum_{h=1}^H \min\cbr{2, \gamma \nbr{\phi(s_{h,t},a_{h,t})}_{\Lambda_{h,t-1}^{-1}}}\\
& \leq \sum_{t=1}^T\zeta_t + \sum_{h=1}^H \sqrt{T\gamma^2} \cdot \sqrt{\sum_{t=1}^T \nbr{\phi(s_{h,t},a_{h,t})}^2_{\Lambda_{h,t-1}^{-1}}}\\
& \leq \sum_{t=1}^T\zeta_t + \sum_{h=1}^H \sqrt{T\gamma^2} \cdot \sqrt{2d\ln(1+T/d)}.
\end{align*}
Here the last step is an application
of~\pref{corr:elliptical-potential}. The first term forms a
martingale, and we know that $|\zeta_t| \leq 2H$. Therefore, by
Azuma's inequality, we have that with probability at least $1-1/TH$
\begin{align*}
\sum_{t=1}^T \zeta_t \leq \sqrt{8TH^2\ln(TH)}.
\end{align*}
Finally, using the definition of $\gamma$, the final regret is upper
bounded by
\begin{align*}
\textrm{Regret}(T) &\leq \order\rbr{H\sqrt{T\ln(TH)} + HK\kappa^{-1}\sqrt{(M+K + d^2\ln((K+\Gamma)TH))\cdot Td\ln(1+T/d)}}\\
& \leq \widetilde{O}\rbr{H\sqrt{d^3 T}}.\tag*\qedhere
\end{align*}

\end{proof}

\section{Tail inequalities}

\begin{lemma}[Azuma's inequality]
Suppose $X_0,X_1,X_2,\cdots,X_N$ form a \emph{martingale} (i.e., $\mathbb E[X_{k+1}|X_1,\cdots,X_k]=X_k$)
and satisfy $|X_{k}-X_{k-1}|\leq c_k$ almost surely. Then for any $\epsilon>0$,
$$
\Pr\left[\big|X_n-X_0\big|\geq \epsilon\right] \leq 2\exp\left\{-\frac{\epsilon^2}{2\sum_{k=1}^Nc_k^2}\right\}.
$$
\label{lem:azuma-hoeffding}
\end{lemma}

\begin{lemma}
\label{lem:self-normalized}
Fix $t,D\in\mathbb N$.
Let $\{\xi_\tau,u_\tau\}_{\tau\leq t}$ be random variables such that $\mathbb E[\xi_\tau|u_1,\xi_1,\cdots,u_{\tau-1},\xi_{\tau-1},u_\tau] = 0$
and $|\xi_\tau|\leq 1$ almost surely.
Let $q: (u,\phi)\mapsto \RR$ be an arbitrary deterministic function satisfying $|q(u,\phi)-q(u,\phi')|\leq C\|\phi-\phi'\|_2$
for all $u,\phi$ and $\phi'$, where $\phi,\phi'\in \RR^D$.
Then for any $\delta\in(0,1)$ and $R>0$, 
\begin{align*}
\Pr\sbr{ \forall \phi \in \BB_D(R):~  \abr{ \sum_{\tau=1}^t \xi_\tau q(u_\tau,\phi)} \leq C + 2\rbr{1+\sqrt{V_q(\phi)}}\sqrt{D\ln(2tR)+\ln(1/\delta)}} \geq 1-\delta,
\end{align*}
where $\BB_D(R)\defeq \{x\in\RR^D: \nbr{x}_2\leq R\}$ and $V_q(\phi) \defeq \sum_{\tau\leq t}q^2(u_\tau,\phi)$.
\end{lemma}
\begin{proof}
Let $\epsilon>0$ be a small precision parameter to be specified later.
Let $\Hcal \subseteq \BB_D(R)$ be a finite $\epsilon$-covering of $\BB_D(R)$ such that $\sup_{x\in\BB_D(R)}\min_{z\in\Hcal}\nbr{x-z}_2\leq\epsilon$.
Using standard covering number arguments, such a covering exists with $\ln|\Hcal|\leq D\ln(2R/\epsilon)$.

For any $\phi\in\BB_D(R)$ let $\phi'\defeq \argmin_{z\in\Hcal}\nbr{\phi-z}_2$.
By definition, $\nbr{\phi-\phi'}_2\leq\epsilon$.
This implies $\abr{\sum_{\tau=1}^t\xi_\tau[q(u_\tau,\phi)-q(u_\tau,\phi')]} \leq Ct\epsilon$ because $|\xi_\tau|\leq 1$ almost surely.
Subsequently, for any $\Delta>0$, 
\begin{align*}
& \Pr\sbr{\exists\phi\in\BB_D(R):~ \abr{\sum_{\tau=1}^t\xi_\tau q(u_\tau,\phi)} > Ct\epsilon + \Delta}
\leq \Pr\sbr{\exists\phi'\in\Hcal:~ \abr{\sum_{\tau=1}^t\xi_\tau q(u_\tau,\phi')} >\Delta}\\
& \leq \sum_{\phi' \in \Hcal}\Pr\sbr{ \abr{ \sum_{\tau=1}^t\xi_\tau q(u_\tau,\phi')} > \Delta},
\end{align*}
where the last inequality holds by the union bound.

For any fixed $\phi' \in \Hcal$, $h(u_\tau,\phi')$ only depends on $u_\tau$, and therefore $\EE[\xi_\tau \mid q(u_\tau,\phi')]=0$ for all $\tau$.
Invoking~\pref{lem:azuma-hoeffding} with $X_\tau \defeq \sum_{\tau'\leq\tau}\xi_{\tau'}q(u_{\tau'},\phi')$ and $c_{\tau'} = |q(u_{\tau'},\phi')|$, we have 
\begin{align*}
\Pr\sbr{\abr{\sum_{\tau=1}^t\xi_\tau q(u_\tau,\phi')} >\Delta} \leq 2\exp\cbr{\frac{-\Delta^2}{2\sum_{\tau\leq t}q^2(u_\tau,\phi')}} = 2\exp\cbr{\frac{-\Delta^2}{2V_q(\phi')}}
\end{align*}
Equating the right-hand side of the above inequality with $\delta'$
and combining with the union bound application, we have
\begin{equation}
\Pr\sbr{\exists\phi\in \BB_d(R): ~ \abr{\sum_{\tau=1}^t\xi_\tau h(u_\tau,\phi)} > Ct\epsilon + \sqrt{2V_q(\phi')\ln(2/\delta')}}
\leq \delta'|\Hcal|.
\end{equation}
Further equating $\delta'=\delta/|\Hcal|$ and using the fact that $\ln|\Hcal|\leq D\ln(2R/\epsilon)$, we have 
\begin{align*}
\Pr\sbr{\exists\phi\in\BB_d(R):~ \abr{\sum_{\tau=1}^t\xi_\tau q(u_\tau,\phi)} > Ct\epsilon + \sqrt{2DV_q(\phi')\ln(2R/\epsilon) + 2V_q(\phi')\ln(1/\delta)}}
\leq \delta.
\end{align*}
Finally, as $\abr{q(u_\tau,\phi')-q(u_\tau,\phi)}\leq\epsilon$, we have $V_q(\phi') \leq 2V_q(\phi) + 2t\epsilon^2$ and so
\begin{align*}
\Pr\sbr{\exists\phi\in\BB_D(R): ~ \abr{\sum_{\tau=1}^t\xi_\tau q(u_\tau,\phi)} > Ct\epsilon + 2\epsilon \sqrt{D t \ln(2R/\epsilon\delta)}+ 2\sqrt{V_q(\phi)(D\ln(2R/\epsilon)+\ln(1/\delta)}}
\leq \delta.
\end{align*}
Setting $\epsilon=1/t$ in the above inequality completes the proof. 
\end{proof}

\bibliography{refs}
\vfill

\end{document}